\pdfoutput=1

\documentclass[11pt]{article}

\usepackage[final]{acl}

\usepackage{times}
\usepackage{latexsym}
\usepackage{amsfonts}
\usepackage{amsmath}
\usepackage{booktabs}
\usepackage{multirow}
\usepackage[T1]{fontenc}

\usepackage{amsthm}

\newtheorem{theorem}{Theorem}[section]
\newtheorem{observation}{Observation}[theorem]
\newtheorem{corollary}{Corollary}[theorem]

\usepackage[utf8]{inputenc}

\usepackage{microtype}

\usepackage{inconsolata}

\usepackage{graphicx}
\usepackage{enumitem}

\usepackage{soul}
\usepackage[most]{tcolorbox}
\usepackage{multicol}
\usepackage{xcolor}
\definecolor{mintleaf}{RGB}{0, 184, 148}
\definecolor{dm-blue-500}{RGB}{0, 69, 177}
\definecolor{dm-purple-500}{RGB}{105,50,230}
\definecolor{mysilver}{RGB}{128,129,128}
\definecolor{my_green}{RGB}{0, 176, 80}
\definecolor{my_yellow}{RGB}{255,165,0}
\definecolor{my_red}{RGB}{255, 0, 0}
\definecolor{self_green}{RGB}{106, 142, 64}
\definecolor{self_red_light}{RGB}{210, 131, 120}
\definecolor{self_red_dark}{RGB}{177, 70, 55}
\definecolor{self_blue}{RGB}{56,92,167}
\definecolor{self_blue_light}{RGB}{74,150,252}
\definecolor{my_blue}{RGB}{49, 133, 155}
\definecolor{case_purple}{RGB}{160, 43, 147}
\definecolor{case_blue}{RGB}{15, 158, 213}

\title{Direct Multi-Turn Preference Optimization for Language Agents}

\author{Wentao Shi$^{1}$\thanks{Both authors contributed equally to this research.}~~Mengqi Yuan$^{1}$\footnotemark[1]~~Junkang Wu$^1$~~Qifan Wang$^2$~~Fuli Feng$^{1}$\thanks{Corresponding Author}\\[10pt]
$^1$University of Science and Technology of China \quad
$^2$Meta AI \\ [5pt]
{\small{\texttt{\{shiwentao123, yuanmengqi\}@mail.ustc.edu.cn \quad wqfcr@fb.com \quad \{jkwu0909,fulifeng93\}@gmail.com}}}
}

\begin{document}
\maketitle
\begin{abstract}
Adapting Large Language Models (LLMs) for agent tasks is critical in developing language agents. Direct Preference Optimization (DPO) is a promising technique for this adaptation with the alleviation of compounding errors, offering a means to directly optimize Reinforcement Learning (RL) objectives. However, applying DPO to multi-turn tasks presents challenges due to the inability to cancel the partition function. Overcoming this obstacle involves making the partition function independent of the current state and addressing length disparities between preferred and dis-preferred trajectories. In this light, we replace the policy constraint with the state-action occupancy measure constraint in the RL objective and add length normalization to the Bradley-Terry model, yielding a novel loss function named DMPO for multi-turn agent tasks with theoretical explanations. Extensive experiments on three multi-turn agent task datasets confirm the effectiveness and superiority of the DMPO loss. The code is available at \url{https://github.com/swt-user/DMPO}.
\end{abstract}

\section{Introduction}

%
Developing generalist agents capable of solving complex tasks has been a central goal in the artificial intelligence community~\cite{DBLP:journals/tmlr/ReedZPCNBGSKSEBREHCHVBF22, DBLP:journals/corr/abs-2404-10179}. Recently, \textit{Language agents}~\cite{DBLP:conf/iclr/YaoZYDSN023} emerge as a prominent research direction, leveraging the considerable potential of Large Language Models to address intricate tasks involving instruction following~\cite{DBLP:conf/nips/Ouyang0JAWMZASR22}, action planning~\cite{DBLP:conf/icml/HuangAPM22}, and tool utilization~\cite{DBLP:conf/nips/SchickDDRLHZCS23}. 
Nevertheless, the substantial disparity between the pretraining task of Large Language Models and the requirements of agent tasks suggests significant potential for future advancements in language agent capabilities.

%
Behavioral Cloning (BC)~\cite{DBLP:journals/neco/Pomerleau91} is a frequently employed approach to bridge the domain gap by fine-tuning LLMs through expert agent trajectories. Recent endeavors in BC~\cite{DBLP:journals/corr/abs-2310-05915, DBLP:journals/corr/abs-2310-12823, DBLP:journals/corr/abs-2311-05657} involve the Supervised Fine-tuning of LLMs on optimal state-action pairs. Although these methods enable swift adaptation of LLMs to agent tasks, BC is notably susceptible to \textit{compounding errors} --- minor errors of the learner accumulate along interactions between the agent and environment, leading to performance deterioration in non-deterministic environments~\cite{DBLP:journals/jmlr/RossGB11}. 

%
In alleviating compounding errors, Direct Preference Optimization~\cite{DBLP:conf/nips/RafailovSMMEF23} has demonstrated remarkable success in the single-turn preference alignment task due to its simple implementation and robustness. DPO optimizes RL objectives by maximizing the likelihood of preferred responses over dis-preferred responses, mitigating the need for continuous interaction with the environment and the training instability commonly associated with traditional RL algorithms~\cite{DBLP:journals/corr/abs-2312-14878, DBLP:journals/corr/abs-2404-01663}. Although there has been an initial endeavor to apply the DPO loss on LLMs for agent tasks~\cite{DBLP:journals/corr/abs-2403-02502}, it encounters suboptimal performance, as it is tailored specifically for the single-turn bandit setting and is ill-suited for multi-turn agent tasks.

\begin{figure}[t]
  \centering
  \includegraphics[width=0.97\linewidth]{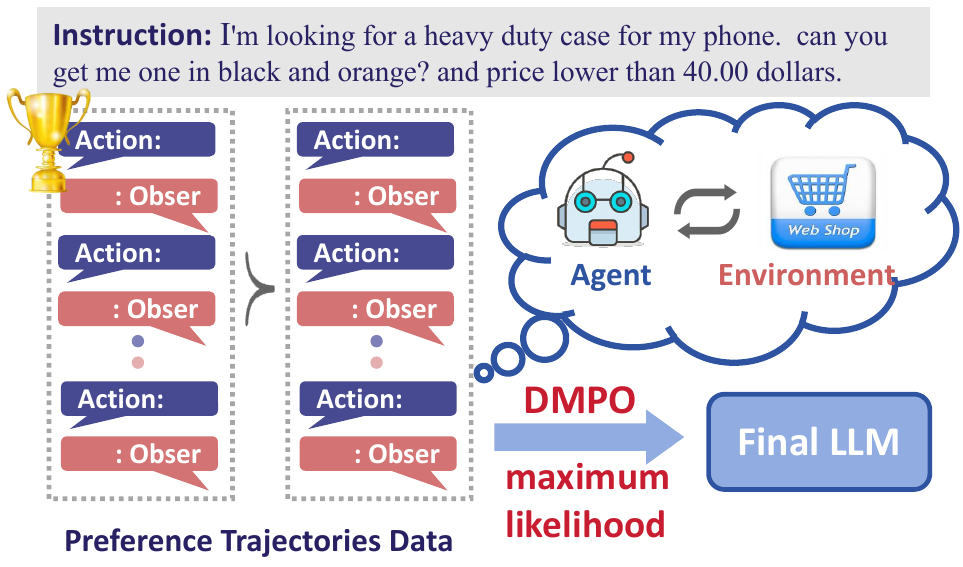}
  \caption{Illustration of DMPO loss, which directly optimizes the RL objective by maximizing the likelihood of the preferred trajectory over the dispreferred trajectory.}
  \label{fig:DMPO}
  \vspace{-5pt}
\end{figure}
%
This work aims to develop a robust loss function capable of directly optimizing RL objectives in multi-turn scenarios. The crux of this pursuit involves eliminating the partition function in the Bradley-Terry (BT) model~\cite{19ff28b9-64f9-3656-ba40-08326a05748e, DBLP:conf/nips/ChristianoLBMLA17}. This entails ensuring the partition function's independence from the current state and neutralizing the impact of the length disparity between preferred and dis-preferred trajectories. To achieve this, we substitute the policy constraint with the state-action occupancy measure (SAOM)~\cite{DBLP:journals/ijon/JohnsonLC00} constraint in the RL objective and introduce length normalization into the BT model. These adjustments culminate in the development of a new and simple loss function DMPO for multi-turn agent tasks. As shown in Figure~\ref{fig:DMPO}, DMPO directly optimizes the RL objective by maximizing the likelihood of preferred ("win") trajectory over dis-preferred ("lose") trajectory. Notably, the SAOM constraint has advantages in mitigating compounding errors compared to the policy constraint~\cite{DBLP:conf/nips/XuLY20, DBLP:conf/corl/GhasemipourZG19}. Furthermore, the derivation offers a theoretical rationale for the efficacy of the length normalization technique in DPO loss~\cite{meng2024simpo}. 

To summarize, our contributions are threefold:
\begin{itemize}[leftmargin=*]
    \item  We introduce a new loss function called DMPO, which directly optimizes RL objectives in multi-turn scenarios, thereby mitigating the compounding errors associated with BC methods.
    \item We provide a theoretical explanation for the efficacy of the length normalization technique, illustrating how it cancels out the partition function in the BT model and improves performance.
    \item Extensive experiments on three multi-turn agent task datasets validate the effectiveness and the superiority of the DMPO loss function.
\end{itemize}

\section{Related Work}
In this section, we first introduce the in-context learning methods and fine-tuning methods of language agents and then review the literature in preference-based RL.
\paragraph{In-Context Learning}
Inspired by the superior in-context learning capabilities of LLMs~\cite{openai2024gpt4}, researchers have designed various instruction prompts for LLMs, equipped with memory modules~\cite{DBLP:journals/corr/abs-2404-13501}, toolkits~\cite{qu2024tool}, and various workflows~\cite{DBLP:journals/corr/abs-2309-02427}, to build language agents for various real-world domains. ReAct~\cite{DBLP:conf/iclr/YaoZYDSN023} incorporates CoT reasoning~\cite{DBLP:conf/nips/Wei0SBIXCLZ22} into action generation. Reflexion~\cite{DBLP:conf/nips/ShinnCGNY23} and PROMST~\cite{DBLP:journals/corr/abs-2402-08702} refine the prompt using environment feedback. However, these in-context learning methods fail to fully exploit the potential of LLMs, since most LLMs are not specifically trained for agent tasks. This work focuses on adapting the LLMs to agent tasks through fine-tuning.

\paragraph{Agent Tuning}
Recent studies, including FireAct~\cite{DBLP:journals/corr/abs-2310-05915}, AgentTuning~\cite{DBLP:journals/corr/abs-2310-12823}, Lumos~\cite{DBLP:journals/corr/abs-2311-05657},  MIMIR~\cite{DBLP:journals/corr/abs-2404-04285}, AUTOACT~\cite{DBLP:journals/corr/abs-2401-05268}, and $\alpha$-UMi~\cite{DBLP:journals/corr/abs-2401-07324} supervised fine-tuning LLMs with self-instruct or expert trajectories. However, such BC approaches suffer from compounding errors when interacting with dynamic environments. Taking a step further, Pangu~\cite{DBLP:journals/corr/abs-2312-14878} and CMAT~\cite{DBLP:journals/corr/abs-2404-01663} utilize RL technologies to further fine-tune the LLMs, which may result in a complex and unstable training procedure. To simplify the procedure, ETO~\cite{DBLP:journals/corr/abs-2403-02502} and EMMA~\cite{DBLP:journals/corr/abs-2311-16714} directly employ the DPO loss~\cite{DBLP:conf/nips/RafailovSMMEF23} to optimize the RL objective for the agent task. Nevertheless, the DPO loss is designed for single-turn bandit settings and is ill-suited for multi-turn scenarios. Along this line, this work extends the DPO loss in multi-turn scenarios and derives the DMPO loss.


\paragraph{Preference-Based RL} In multi-turn scenarios, preference-based RL typically starts by explicitly learning a reward function from preference data and then optimizing it~\cite{DBLP:journals/ml/FurnkranzHCP12, DBLP:conf/nips/ChristianoLBMLA17, DBLP:conf/corl/HejnaS22, DBLP:journals/corr/abs-2107-09251}. However, this two-stage learning process presents challenges regarding training efficiency and instability. This work instead presents a single-stage policy learning approach using DMPO loss that directly optimizes a policy to satisfy preferences. While IPL~\cite{DBLP:conf/nips/HejnaS23} and CPL~\cite{DBLP:journals/corr/abs-2310-13639} share a similar idea with our work in eliminating the reward learning stage, their loss functions are limited to trajectory pairs of equal length, significantly restricting their applicability.

\section{Preliminary}
In this section, we present multi-turn agent task formulation and briefly introduce Direct Preference Optimization (DPO) loss.
\subsection{Task Description}
The agent task can be formulated as a Markov decision process (MDP). A MDP is a 5-tuple ($\mathcal{S}, \mathcal{A}, \mathcal{T}, \mathcal{R}, \gamma)$, where $\mathcal{S}$ denotes the state space, $\mathcal{A}$ denotes action space, $\mathcal{T}$ denotes dynamic transition function $\mathcal{S}\times \mathcal{A}\rightarrow \mathcal{S}$, $\mathcal{R}$ denotes reward function $\mathcal{S}\times \mathcal{A} \rightarrow[0,1]$, and $\gamma\in[0,1)$ is the discount factor. The goal for the agent is to choose actions at each time step that maximize the expected future discounted reward $\mathbf{E}\left[ \sum_{t=0}^{T-1} \gamma^t r(s_t, a_t) \right]$, where $T$ is the trajectory length. 

In the language agent setting~\cite{DBLP:journals/corr/abs-2312-14878}, the state space and action space are both subsets of the language space. For the initial state $s_0\in \mathcal{S}$, it contains the task instruction and prompt. At each time step $t$, LLMs generate action $a_t$ according to the policy $\pi_\theta(a_t|s_t)$ with the parameter $\theta$. Then the environment will return dynamic feedback $o_t$ and transport the state into $s_{t+1}$. Note that the new state $s_{t+1}$ is just a simple combination of $s_t$, $a_t$, and $o_t$, and the trajectory $\tau=(s_0,a_0, s_1, a_1, \cdots, s_T,a_T)$. 

\subsection{Direct Preference Optimization}

The aim of the DPO loss is to directly optimize RL objectives with KL divergence constraints on the policy function:
\begin{multline}
   \max_{\pi_\theta} \mathbb{E}_{\tau}
      [\sum_{t=0}^{T-1}\gamma^t r(s_t,a_t)] \\
    - \beta \mathbb{D}_{KL}[\pi_\theta(a_t|s_t)||\pi_{ref}(a_t|s_t)],
    \label{equ:RL objective}
\end{multline}
where $\mathbb{E}$ is the expectation function, $\mathbb{D}_{KL}[\cdot||\cdot]$ denotes the KL divergence between two distributions, $\pi_{ref}$ denotes a reference policy, and the 
$\beta$ is a parameter controlling the deviation from the base reference policy $\pi_{ref}$. The DPO loss is tailored for the single-turn preference alignment setting, where the trajectory length ($T$) is limited to 1. \par Notably, the reward function is learned through the Bradley-Terry (BT) model~\cite{19ff28b9-64f9-3656-ba40-08326a05748e, DBLP:conf/nips/ChristianoLBMLA17}:
\begin{equation}
p(a_0^w \succ a_0^l|s_0)=\frac{\exp(r(s_0,a_0^w))}{\exp(r(s_0,a_0^w))+\exp(r(s_0,a_0^l))},
\label{equ:BT_model}
\end{equation}
which gives the probability that the “win” action $a_0^w$ is preferred to the “lose” action $a_0^l$ given the state $s_0$.

Then DPO leverages the established closed-form solution for the single-turn formulation of the reinforcement learning problem in Eq~\eqref{equ:RL objective} presented in~\cite{DBLP:conf/aaai/ZiebartMBD08, DBLP:phd/us/Ziebart18}:
\begin{equation}
    \pi^*(a|s) = \frac{1}{Z(s)}\pi_{ref}(a|s)e^{r(s,a)},
    \label{equ:optimal_policy}
\end{equation}
where $\pi^*$ denotes the optimal policy and $Z(s)$ denotes the partition function that normalizes it. We can easily rearrange Eq~\eqref{equ:optimal_policy} and substitute it into Eq~\eqref{equ:BT_model} to get the BT model over policy:
\begin{multline}
    p(a_0^w \succ a_0^l|s_0) = \\
    \sigma\left( \beta \log \frac{\pi_\theta(a_0^w|s_0)}{\pi_{ref}(a_0^w|s_0)} - \beta \log \frac{\pi_\theta(a_0^l|s_0)}{\pi_{ref}(a_0^l|s_0)}\right),
\end{multline}
where the partition function $Z(s)$ is canceled from the BT model and $\sigma$ is the sigmoid function. The DPO loss obtains the optimal policy $\pi_\theta^*$ by maximizing the likelihood:
\begin{equation}
    \mathcal{L}_{DPO} = - \mathbb{E}_{(s_0, a_0^w, a_0^l)\sim D} \log \left[p(a_0^w \succ a_0^l|s_0) \right],
\end{equation}
where $D$ represents the preference dataset. Nonetheless, such concise and elegant derivations are only suitable for single-turn preference optimization tasks. As shown in Eq~\eqref{equ:optimal_policy}, the partition function $Z(s)$ is dependent on the current state $s$, which precludes its cancellation under the policy constraint in the multi-turn setting.

\section{Method}
In this section, we will outline the definition and benefits of the state-action occupancy measure. Subsequently, we will introduce two adjustments to derive the DMPO loss. Finally, we will delve deeper into the analysis of the DMPO loss.

\subsection{State-Action Occupancy Measure}
The discounted state-action occupancy measure $d^\pi(s, a)$ of a policy $\pi$ describes the distribution of state-action pairs that an agent visits in the space with policy $\pi$:
\begin{equation}
\begin{split}
d^{\pi}(s,a)=\frac{1-\gamma}{1-\gamma^T}\sum_{t=0}^{T-1}\gamma^t\mathbb{P}(s_t=s, a_t=a|\pi),
\end{split}
\end{equation}
where $\mathbb{P}(\cdot)$ denotes the probability and the coefficient $(1-\gamma)/(1-\gamma^T)$ is used to normalize the probability distribution. 

\begin{figure}[t]
  \centering
  \includegraphics[width=\linewidth]{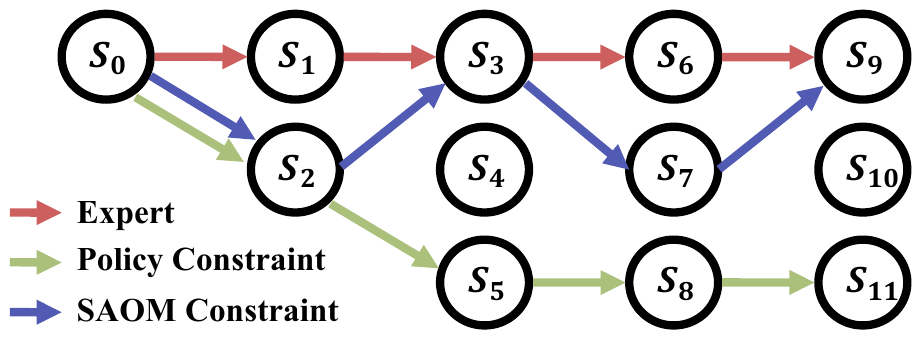}
  \caption{Illustration of expert trajectories and trajectories learned under the constraints of policy and state-action occupancy measure. }
  \label{fig:compounding_error}
\end{figure} 

First, we will provide an intuitive explanation of how the SAOM constraint can reduce the compounding error. In imitation learning, the conventional SFT learning objective aims to minimize the KL divergence between the expert policy and the current policy:
\begin{equation}
    \begin{split}
        &\min_{\pi_\theta} \mathbb{E}_{(s,a)\sim d^E}[\mathbb{D}_{KL}(\pi_{E}(a|s)||\pi_\theta(a|s)] \\
        =  &-\max_{\pi_\theta} \mathbb{E}_{(s,a)\sim d^{E}}[\log(\pi_\theta(a|s)],
    \end{split}
\end{equation}
where $\pi_{E}$ is the expert policy and $d^{E}$ is the SAOM with policy $\pi_{E}$. As shown in Figure~\ref{fig:compounding_error}, the trajectories learned under policy constraints are susceptible to significant compounding error. This vulnerability stems from the fact that expert datasets are unable to comprehensively cover all possible states. Consequently, the SFT loss leads the model to choose random actions in states that are not represented in the expert datasets. As a result, the model gradually deviates from the expert trajectories after the initial error, illustrating the phenomenon known as compounding error. \par 
To alleviate the compounding error, subsequent imitation learning research such as~\cite{DBLP:conf/icml/PieterN04, DBLP:conf/corl/GhasemipourZG19, DBLP:conf/nips/HoE16} employ the SAOM constraint:
\begin{equation}
    \min_{\pi_\theta} \mathbb{E}_{(s,a)\sim d^E}[\mathbb{D}_{(\cdot)}(d^{\pi_\theta}(a|s)||d^{\pi_{E}}(a|s))],
\end{equation}
where different approaches utilize different distribution distance measures $\mathbb{D}_{(\cdot)}$. The strength of SAOM constraint lies in its ability to steer action selection towards distributions that closely mimic expert state-action pairs, especially in unexplored states within the expert datasets. Illustrated in Figure~\ref{fig:compounding_error}, at state $s_2$, policy constraints lead the model to choose actions uniformly, whereas SAOM constraints aim to lead the model toward actions that bring the next state back onto the expert trajectory. This effectively mitigates compounding errors and enhances the cumulative reward.

\subsection{DMPO}\label{section:DMPO}
Inspired by imitation learning, we substitute the policy constraint with the SAOM constraint in Eq~\eqref{equ:RL objective} and get the following RL objective:
\begin{multline}
    \max_{\pi_\theta} \mathbb{E}_{(s,a)\sim d^{\pi_\theta}(s,a)}[r(s,a)] \\
    - \beta \mathbb{D}_{KL}[d^{\pi_\theta}(s,a)||d^{\pi_{ref}}(s,a)],
    \label{equ:s-a RL objective}
\end{multline}
where $\pi_{ref}$ represents the reference policy. Similar to~\cite{DBLP:conf/nips/RafailovSMMEF23}, it is straightforward to show that the optimal solution to the RL objective in Eq~\eqref{equ:s-a RL objective} takes the form:
\begin{equation}
    d^{\pi^{*}}(s,a) = \frac{1}{Z} d^{\pi_{ref}}(s,a)\exp(\frac{1}{\beta}r(s,a))
    \label{equ:s-a solution},
\end{equation}
where $\pi^*$ represents the optimal policy, $Z$ is the partition function that normalizes the probability. It's noteworthy that as $d^{\pi}(s,a)$ is a function of $(s,a)$ pairs, normalizing it results in the partition functions $Z$ being independent of the current state $s$. Consequently, $Z$ remains constant for all $(s,a)$ pairs, providing us with the opportunity to eliminate them. Easily, we can rearrange Eq~\eqref{equ:s-a solution} into:
\begin{equation}
    r(s,a) = \beta \log \frac{d^{\pi^*}(s,a)}{d^{\pi_{ref}}(s,a)} + \beta \log Z.
    \label{equ:s-a rearrange}
\end{equation}
Similar to Eq~\eqref{equ:BT_model}, we learn the reward function for multi-turn scenarios through the BT model:
\begin{multline}
    p(\tau^w\succ \tau^l|s_0) = \\
    \sigma\left(\sum_{t=0}^{T_w-1}\gamma^tr(s_t^w,a_t^w) - \sum_{t=0}^{T_l-1} \gamma^tr(s_t^l,a_t^l)\right),
    \label{equ:traj_BT}
\end{multline}
where $\tau^w$ and $\tau^l$ represent the "win" and "lose" trajectories respectively, $T_w$ and $T_l$ represent the "win" and "loss" trajectory length respectively. However, since $T^w\neq T^l$, the partition function $Z$ cannot be canceled directly in Eq~\eqref{equ:traj_BT}. \par 
To overcome this obstacle, we introduce the length normalization technique to Eq~\eqref{equ:traj_BT}:
\begin{multline}
    p(\tau^w\succ \tau^l|s_0) = 
    \sigma\left(\frac{1-\gamma}{1-\gamma^{T_w}}\sum_{t=0}^{T_w-1}\gamma^tr(s_t^w,a_t^w) \right.\\  \left. - \frac{1-\gamma}{1-\gamma^{T_l}}\sum_{t=0}^{T_l-1} \gamma^tr(s_t^l,a_t^l)\right).
    \label{equ:traj_BT_norm}
\end{multline}
In this way, we can eliminate the partition function $Z$ in Eq~\eqref{equ:traj_BT_norm} by substituting the reward function $r(s,a)$ in Eq~\eqref{equ:s-a rearrange}. Then we  maximize the likelihood and obtain:
\begin{multline}
    L_{DMPO}=-\mathbb{E}_{(s_0,\tau^w,\tau^l)\sim D}\log \sigma \\ 
    \left[ \frac{1-\gamma}{1-\gamma^{T_w}}\sum_{t=0}^{T_w-1} \beta \gamma^t \log \frac{d^{\pi_\theta}(s_t^w,a_t^w)}{d^{\pi_{ref}}(s_t^w,a_t^w)} - \right.\\ 
    \left.  \frac{1-\gamma}{1-\gamma^{T_l}}\sum_{t=0}^{T_l-1} 
     \beta \gamma^t \log\frac{d^{\pi_\theta}(s_t^l,a_t^l)}{d^{\pi_{ref}}(s_t^l,a_t^l)}
     \right],
     \label{equ:DMPO_1}
\end{multline}
where the $d^\pi(s_t,a_t)$ can be further written as:
\begin{multline}
    d^{\pi}(s=s_t^w,a=a_t^w)= \gamma^t\cdot P(s_0) \cdot \\
    \prod_{k=0}^{t-1}\pi(a_k^w|s_k^w)P(s_{k+1}^w|s_k^w,a_k^w),
    \label{equ:DMPO_2}
\end{multline}
where $P(s_0)$ represents the probability of the initial state $s_0$ and $P(s_{k+1}|s_{k},a_k)$ denotes the transition functions. In general, obtaining the SAOM $d^\pi(s_t,a_t)$ is challenging because we do not know the transition function $P(s_{k+1}|s_{k},a_k)$ in dynamic environments. However, in Eq~\eqref{equ:DMPO_1} we simply calculate the ratio between the current SAOM $d^{\pi_\theta}(s_t,a_t)$ and the reference SAOM $d^{\pi_{ref}}(s_t,a_t)$. It is important to note that the transition function remains consistent for both, allowing for cancellation. By substituting the Eq~\eqref{equ:DMPO_2} into Eq~\eqref{equ:DMPO_1}, we can obtain the DMPO loss function:
\begin{multline}
    L_{DMPO}=-\mathbb{E}_{(s_0,\tau^w,\tau^l)\sim D}\log \sigma\\ \left[ \sum_{t=0}^{T_w-1} \beta \phi(t,T_w) \log \frac{\pi_\theta(a_t^w|s_t^w)}{\pi_{ref}(a_t^w|s_t^w)} \right.\\
    \left.  - \sum_{t=0}^{T_l-1} \beta \phi(t,T_l) \log \frac{\pi_\theta(a_t^l|s_t^l)}{\pi_{ref}(a_t^l|s_t^l)}
     \right],
     \label{equ:DMPO_3}
\end{multline}
where the discount function $\phi(t,T)=\gamma^t\cdot(1-\gamma^{T-t})/(1-\gamma^{T})$. It's noteworthy that DMPO reweights state-action pairs at various steps using a discount function $\phi(t,T)$.

\subsection{In-Depth Analysis}
In this subsection, we will explore the advantages of the DMPO loss and present some lemmas and observations.

\begin{corollary} 
The DMPO loss assigns higher weights to state-action pairs at early steps, where the weight is related to discount factor $\gamma$.
\end{corollary}
\begin{proof}
To prove the lemma, we analyze the gradient of the loss function $L_{DMPO}$ according to $\theta$:
\begin{multline}
    \nabla_\theta L_{DMPO} = -\mathbb{E}_{(s_0,\tau^w,\tau^l)\sim D} \sigma[\Phi(\tau^l) - \Phi(\tau^w)] \\
    \left[   \sum_{t=0}^{T_w-1} \beta \phi(t,T_w) \nabla_\theta \log \pi_\theta(a_t^w|s_t^w)         \right.\\
    \left.  - \sum_{t=0}^{T_l-1} \beta \phi(t,T_l) \nabla_\theta \log \pi_\theta(a_t^l|s_t^l)         \right],
    \label{equ:gradient}
\end{multline}
where function $\Phi(\tau)=\sum_{t=0}^{T-1}\beta \phi(t,T)\log \frac{\pi_\theta(a_t|s_t)}{\pi_{ref}(a_t|s_t)}$ and $\phi(t,T)=\gamma^t\cdot(1-\gamma^{T-t})/(1-\gamma^{T})$. The discount function $\phi(t,T)$ decreases as $t$ increases and is related to the discounted factor $\gamma$. This completes the proof.
\end{proof}
\begin{corollary}
The DMPO loss degenerates into the single-turn DPO loss when the discount factor $\gamma$ approaches zero.
\end{corollary}
\begin{proof}
When $\gamma$ equals 0, the function $\phi(t,T)$ is 1 at $t=0$, and 0 otherwise, which is equivalent to a single-turn DPO loss.
\end{proof}
Based on the analysis above, we have the following observations:\par
\begin{observation}
Similar to the DPO loss, the DMPO loss increases the likelihood of the preferred trajectories $\tau_w$ and decreases the likelihood of the dispreferred trajectories $\tau_l$.
\end{observation}
\begin{observation}
If the reward $\Phi(\tau_l)$ of dispreferred trajectory is estimated higher by the policy $\pi_\theta$, the weight $\sigma[\Phi(\tau^l) - \Phi(\tau^w)]$ will be larger.
\end{observation} 

\paragraph{Length Normalization Explanation} In SimPO~\cite{meng2024simpo}, the effectiveness of the length normalization technique was empirically demonstrated. However, a theoretical explanation was not provided. Our derivation shows that it assists in eliminating the partition function. Without length normalization in Eq~\eqref{equ:traj_BT_norm}, a length-dependent bias term arises in the BT model, degrading model performance as the disparity in trajectory lengths between preferred and dispreferred samples increases.

\paragraph{Further Discussion}
As discussed in Section~\ref{section:DMPO}, the optimal solution to the RL objective in Eq~\eqref{equ:s-a RL objective} takes the form shown in Eq~\eqref{equ:s-a solution}. However, it is contended that achieving the optimal solution may not always be feasible when dealing with an arbitrary reward function $r(s,a)$ within the context of a language agent setting. This limitation arises due to the definition of the new state $s_{t+1}$ as a composite of $s_t$, $a_t$, and $o_t$, which introduces an inherent constraint on the transition function between states. In general, in multi-turn dynamic environments, no loss function can rigorously optimize the RL objective, and the DMPO loss serves as a good approximation. In many cases, the DMPO loss can precisely optimize the RL objective in Eq~\eqref{equ:s-a RL objective}.

\section{Experiments}

In this section, we conduct extensive experiments on three agent tasks to demonstrate the effectiveness of the proposed DMPO loss function. Our experiments aim to address the following questions: \par
\noindent $\bullet$ \textbf{RQ1:} Can the DMPO loss function exhibit robustness to noisy training trajectories data and mitigate compounding errors? \par
\noindent $\bullet$ \textbf{RQ2:} How does the DMPO loss function perform compared to other baselines? \par
\noindent $\bullet$ \textbf{RQ3:} What is the impact of the discount factor $\gamma$ and the trajectory length on the DMPO loss?

\subsection{Experiment Setting}

\paragraph{Datasets} 

\begin{table}[t]
\centering
\tabcolsep=3pt
\resizebox{1\linewidth}{!}{
\begin{tabular}{lccc}
\toprule
\textbf{Dataset} & WebShop & ScienceWorld & ALFWorld  \\
\midrule
Train        &  1938    &   1483    & 3321 \\
Test-Seen    & 200 & 194 & 140 \\
Test-UnSeen  & - & 241 & 134 \\
\bottomrule
\end{tabular}
}
\caption{
Statistics of three agent datasets. ``Train'', ``Test-Seen'', and ``Test-Unseen'' refer to the number of tasks in each set respectively.
}
\vspace{-5pt}
\label{tab:dataset}
\end{table}

Following prior work~\cite{DBLP:journals/corr/abs-2403-02502},  we conduct experiments on three representative agent datasets, including WebShop~\cite{DBLP:conf/nips/Yao0YN22}, ScienceWorld~\cite{DBLP:conf/emnlp/WangJCA22}, and
ALFWorld~\cite{DBLP:conf/iclr/ShridharYCBTH21}. \par
\noindent $\bullet$ WebShop is a simulated shopping website environment where agents find and purchase products according to specifications provided in a natural language instruction. The final reward $r\in[0,1]$ is calculated based on how closely the purchased products match the specified criteria.\par
\noindent $\bullet$ ScienceWorld is an interactive text environment that tests agents' scientific reasoning abilities in elementary science experiments with 10 task types. The final reward $r\in [0,1]$ is computed based on the number of subgoals the agent successfully accomplishes within each task.\par
\noindent $\bullet$ ALFWorld is a simulated text-based environment that enables agents to complete embodied household tasks from the ALFRED benchmark~\cite{DBLP:conf/cvpr/ShridharTGBHMZF20}. The final binary rewards signify the completion status of the task.\par

All three environments can be formally described as MDP and conducted by language agents. The statistical details of our datasets are outlined in Table~\ref{tab:dataset}. Following~\cite{DBLP:journals/corr/abs-2403-02502}, in addition to the in-distribution ``seen'' test sets, both ScienceWorld and ALFWorld include ``unseen'' test sets that include out-of-distribution tasks. These additional test sets enable us to evaluate the generalization capabilities of different agents.

\paragraph{Training Setting}
We assess the robustness and effectiveness of the DMPO loss function by employing two distinct training scenarios: Noisy setting and Clean setting. Following~\cite{DBLP:journals/corr/abs-2403-02502}, we adopt the experts' trajectories as the "win" trajectories to form preference trajectory data in both noisy setting and clean setting. Initially, we utilize the LLMs, which have been fine-tuned with expert trajectories, to generate new trajectories on the training set. We observe that the LLMs have a tendency to generate trajectories with repeated actions or meaningless words. In the noisy setting, these noisy trajectories are used as "lose" trajectories for preference data. Conversely, in the Clean setting, we eliminate the noisy trajectories and employ the remaining ones as "lose" trajectories for preference data.

\begin{table*}[t]
\centering
\tabcolsep=15pt
\resizebox{1\linewidth}{!}{
\begin{tabular}{@{}lccccc@{}}
\toprule
\multirow{2}{*}{\textbf{Method}} & \multirow{2}{*}{\textbf{WebShop}} & \multicolumn{2}{c}{\textbf{ScienceWorld}} & \multicolumn{2}{c}{\textbf{ALFWorld}} \\
\cmidrule(l){3-4} \cmidrule(l){5-6}
& & Seen & Unseen & Seen & Unseen \\
\midrule

Llama-2-7B-Chat + DPO & 0.641 \scriptsize{${\pm}$} 0.002 & 0.601 \scriptsize{${\pm}$} 0.004 & 0.576 \scriptsize{${\pm}$} 0.001 & \textbf{0.474\scriptsize{${\pm}$}0.004} & 0.540 \scriptsize{${\pm}$} 0.005 \\
Llama-2-7B-Chat + DMPO & \textbf{0.666 \scriptsize{${\pm}$} 0.007} & \textbf{0.619 \scriptsize{${\pm}$} 0.003} & \textbf{0.584 \scriptsize{${\pm}$} 0.005} & 0.433 \scriptsize{${\pm}$} 0.004 & \textbf{0.550 \scriptsize{${\pm}$} 0.004}\\
\midrule
Mistral-7B-Instructv0.2 + DPO & 0.637 \scriptsize{${\pm}$} 0.007 &  0.700 \scriptsize{${\pm}$} 0.003 & 0.629 \scriptsize{${\pm}$} 0.008 & \textbf{0.745 \scriptsize{${\pm}$} 0.004} & 0.883 \scriptsize{${\pm}$} 0.004  \\
Mistral-7B-Instructv0.2 + DMPO & \textbf{0.643 \scriptsize{${\pm}$} 0.008} &  \textbf{0.708 \scriptsize{${\pm}$} 0.015}  & \textbf{0.651 \scriptsize{${\pm}$} 0.004} & 0.742 \scriptsize{${\pm}$} 0.012 & \textbf{0.888 \scriptsize{${\pm}$} 0.000} \\
\bottomrule
\end{tabular}
}
\caption{
Noisy setting: The average reward of different base LLMs on three agent datasets. "Seen" denotes in-distribution test sets, while "Unseen" denotes out-of-distribution test sets. The results are averaged with three distinct random seeds. The best results for each base model are highlighted in bold.}
\label{tab:noisy}
\vspace{-5pt}
\end{table*}

\begin{table}[t]
\centering
\tabcolsep=15pt
\resizebox{1\linewidth}{!}{
\begin{tabular}{@{}lccc@{}}
\toprule
\multirow{2}{*}{\textbf{Method}} & \multirow{2}{*}{\textbf{WebShop}} & \multicolumn{2}{c}{\textbf{ScienceWorld}}  \\
\cmidrule(l){3-4} 
& & Seen & Unseen  \\
\midrule
GPT-4* & 0.632 & 0.648 & 0.644 \\
GPT-3.5-Turbo* & 0.624 & 0.165 & 0.130\\
\midrule
Base* & 0.179 & 0.380 & 0.310  \\
Best-of-N* & 0.638 & 0.702 & 0.576 \\
RFT* & 0.636 & 0.716 & 0.543  \\
PPO* & 0.642 & 0.594 & 0.517  \\
\midrule

SFT & 0.631 & 0.568 & 0.560  \\
ETO & 0.698\scriptsize{${\pm}$}0.003 & 0.685\scriptsize{${\pm}$}0.004 & 0.611\scriptsize{${\pm}$}0.003 \\
\midrule
DMPO & \textbf{0.701\scriptsize{${\pm}$}0.003} & \textbf{0.724\scriptsize{${\pm}$}0.005} & \textbf{0.617\scriptsize{${\pm}$}0.002}  \\
\bottomrule
\end{tabular}
}
\caption{
Clean setting: The average reward of different methods on two agent datasets based on Llama-2-7B-Chat. The best results of tuning methods are highlighted in bold. *Results are taken from~\cite{DBLP:journals/corr/abs-2403-02502}.
}
\vspace{-10pt}
\label{tab:clean}
\end{table}

\paragraph{Parameter Settings}
In this work, we utilize two different base models Llama-2-7B-Chat~\cite{touvron2023llama} and Mistral-7B-Instruct-v0.2~\cite{DBLP:journals/corr/abs-2310-06825} to build language agents. Following~\cite{DBLP:journals/corr/abs-2403-02502}, we utilize the AdamW optimizer. When supervised fine-tuning the base models to get the reference model, we set the batch size to 64. The learning rate is selected from \{1e-5, 2e-5, 3e-5\} with 3\% warm up and a cosine scheduler. When refining the agents with DMPO loss function, we set the batch size to 32 and tune the hyperparameters $\beta$ and $\gamma$ within the ranges of \{0.1, 0.2, 0.3, 0.4, 0.5, 0.6, 0.7, 0.8, 0.9 \} and  \{0.1, 0.2, 0.3, 0.4, 0.5, 0.6, 0.7, 0.8, 0.9, 0.99\} respectively. We conduct all experiments on 8 NVIDIA A100 GPUs.

\paragraph{Evaluation Setting} Following~\cite{DBLP:journals/corr/abs-2403-02502}, we evaluate all methods using the ReAct-style interaction format~\cite{DBLP:conf/iclr/YaoZYDSN023}, which generates both reasoning traces and actions in an interleaved manner. For each task, we add 1-shot examples for each task, which can be found in~\cite{DBLP:journals/corr/abs-2403-02502}. Unless otherwise stated, we set the decoding generate temperature as 0.0.

\subsection{Noisy Setting Results (RQ1)}

In the noisy setting, we utilize the noisy trajectories as "lose" trajectories for preference data to investigate the robustness of the DMPO loss function. As shown in Table~\ref{tab:noisy}, we evaluate the DMPO loss function with two different base models on two representative agent tasks and observe that: \par
\noindent $\bullet$ In all Unseen test sets and most Seen test sets for both base models, the DMPO loss function outperforms the DPO loss function. This superiority stems from DMPO assigning greater importance to initial state-action pairs, prioritizing high-quality expert actions from the early stages, and reducing the influence of noisy "lose" actions in later stages. This mitigates the influence of noise, endowing the model with enhanced generalization capabilities. Meanwhile, the DPO loss is not appropriate for multi-turn settings and cannot cancel out the partition function in the BT model, thereby resulting in its inferior performance. \par
\noindent $\bullet$ The performance of Mistral-7B-Instruct-v0.2 is significantly better than that of Llama-2-7B-Chat on Scienceworld and AlfWorld. This observation suggests a positive correlation between the effectiveness of the base model and its performance enhancement after fine-tuning for agent tasks using the DMPO loss function.

\subsection{Clean Setting Results (RQ2)}
In clean setting, we filter out the noisy trajectories and select high-quality trajectories as the "lose" trajectories for preference data, enabling us to utilize the DMPO loss function fully.
\paragraph{Baselines} Following~\cite{DBLP:journals/corr/abs-2403-02502}, we compare our models trained by DMPO loss function with the following representative baselines. 1) Base: default LLM without tuning. 2) SFT: LLM fine-tuned through supervised learning on expert trajectories. 3) Best-of-N: This approach involves using an SFT-based agent for sampling and selecting the trajectory with the highest reward out of N samples. Here, N is specified as 10. 4) RFT (Rejection sampling Fine-Tuning)~\cite{DBLP:journals/corr/abs-2308-01825}: This approach augments the expert trajectory dataset by incorporating successful trajectories and subsequently trains the agent on the augmented dataset. 5) PPO (Proximal Policy Optimization)~\cite{DBLP:journals/corr/SchulmanWDRK17} directly optimize RL objectives to maximize the cumulative rewards. 6) ETO (Exploration-based Trajectory Optimization)~\cite{DBLP:journals/corr/abs-2403-02502} iteratively explores the environment to enhance the training preference data and utilizes DPO loss to learn from preference data.

\paragraph{Results} Based on the Llama-2-7B-Chat model, we show the comparison results under clean setting in Table~\ref{tab:clean}. Notably, we observe that: \par

\noindent $\bullet$ All fine-tuning methods significantly outperform the base model on both datasets, with improvements of at least 49\%. On Webshop, they even surpass the performance of advanced closed-source LLMs. This underscores the significant gap between the pre-training tasks of LLMs and the agent tasks. By fine-tuning LLMs, language agents exhibit substantial potential for improvement. \par

\noindent $\bullet$ The model trained using DMPO loss achieved optimal performance on both datasets, highlighting the effectiveness of DMPO loss in learning from preference data. The improvement over the SFT model suggests that DMPO reduces the compounding errors, resulting in higher rewards. \par
\noindent $\bullet$ The model trained using DMPO loss exhibits substantial performance improvements compared to the noisy setting, achieving an average increase of 5.2\% on Webshop and 11.3\% on Scienceworld. This highlights the importance of selecting high-quality "lose" trajectories in constructing preference data, as opting for such trajectories yields superior performance. \par

\subsection{Ablation Study (RQ3)}

\begin{figure}[t]
  \centering
  \includegraphics[width=0.95\linewidth]{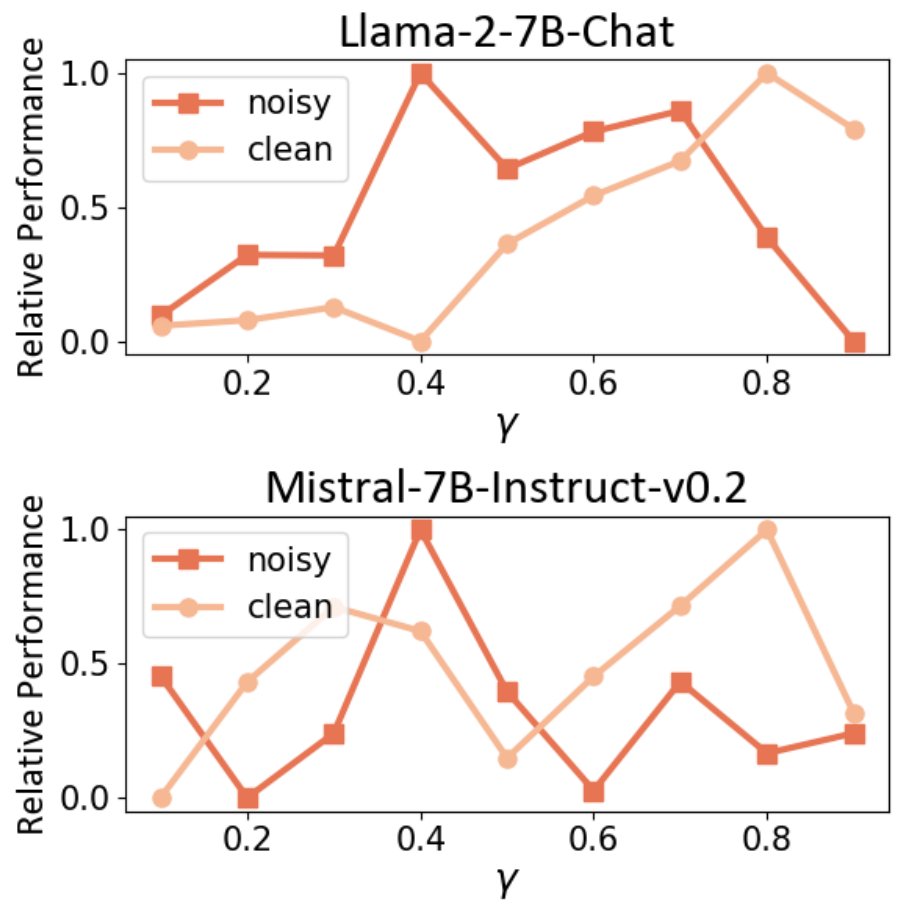}
  \caption{The effect of hyperparameter $\gamma$ on the relative performance of the model trained with DMPO loss on the WebShop dataset in both noisy and clean settings. }
  
  \label{fig:gamma_hyperparameter}
\end{figure}

\paragraph{Hyperparamter Analysis} To verify the impact of reweight function $\phi(t,T)$ in Eq~\eqref{equ:gradient}, we tune the the hyperparameter $\gamma$ on WebShop and present the results in Figure~\ref{fig:gamma_hyperparameter}. Our findings reveal that both base models achieve optimal performance with a smaller $\gamma$ in the noisy setting and a larger $\gamma$ in the clean setting. According to Eq~\eqref{equ:gradient}, a smaller $\gamma$ implies that the DMPO loss assigns reduced weight to the state-action pairs in later steps. This indicates that DMPO can balance the impact of noise by adjusting the parameter $\gamma$. When faced with noisy "loss" trajectories, selecting a smaller $\gamma$ can help alleviate noise impact. Conversely, when dealing with high-quality "loss" trajectories, a larger gamma can be selected to better learn strategies from the state-action pairs in later steps.

\begin{figure}[t]
  \centering
  \includegraphics[width=0.95\linewidth]{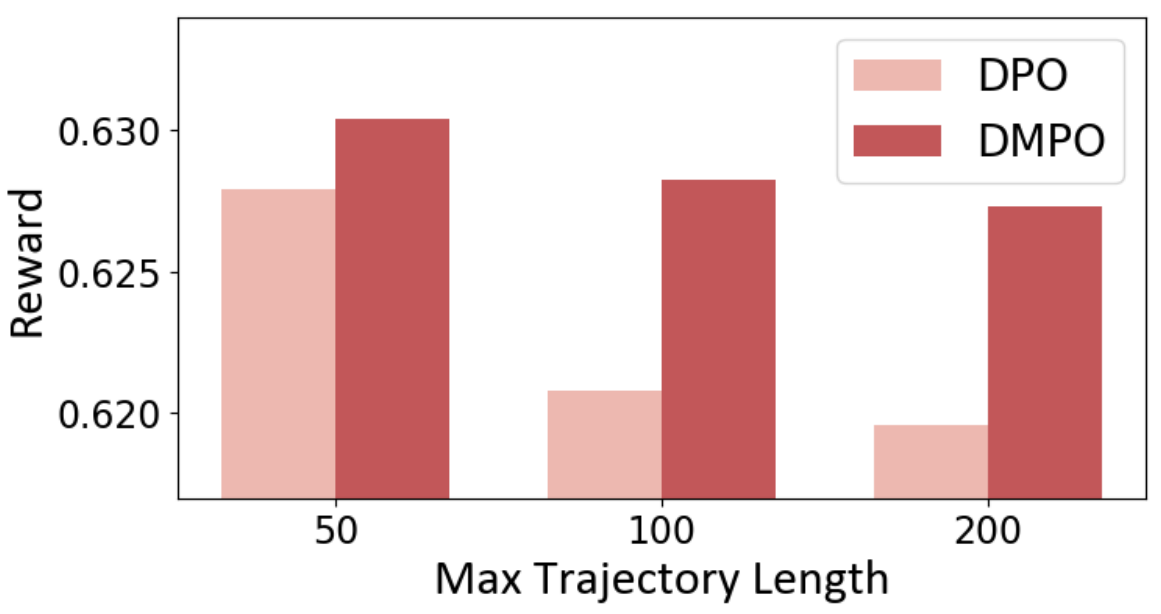}
  \caption{The effect of "loss" trajectories length on the performance of the model trained with DPO and DMPO loss in the noisy setting on ScienceWorld. The base model is Mistral-7B-Instruct-v0.2. }
  \label{fig:length_ablation}
  
\end{figure}

\paragraph{Length Analysis} To examine the impact of trajectory length on model performance, we conducted an experiment by categorizing the noisy trajectories into three groups based on their maximum length. We ensure that the number of preference data in each group is the same. As shown in Figure~\ref{fig:length_ablation}, we observe that the performance of the model trained with DPO loss function decreases rapidly as the length of noisy "loss" trajectories increases. In contrast, the model trained with the DMPO loss function exhibits robustness against noisy "loss" trajectory length. This is attributed to the length normalization employed in the DMPO loss, which mitigates the influence of inconsistent lengths between "win" and "lose" trajectories.


\section{Conclusion}

In this work, we propose a simple and robust loss function DMPO loss, which directly optimizes the RL objective for multi-turn agent tasks. By substituting the policy constraint with the SAOM constraint and introducing the length normalization into BT model, we eliminate the partition function in the BT model and derive the DMPO loss function. The SAOM constraint has played a pivotal role in mitigating compounding errors. Meanwhile, this derivation offers a theoretical rationale for the efficacy of the length normalization technique. Extensive experiments on three agent datasets demonstrate the effectiveness of DMPO loss, highlighting its capability to reduce compounding errors and its resilience to trajectory length disparity.\par

\section{Limitation}
This paper primarily focuses on issues when fine-tuning LLMs on the agent tasks and derives a simple and robust loss function. However, our study has several limitations: 1) We solely concentrate on turn-wise task formulation which results in sparse rewards for LLMs. Exploring token-wise task formulation as suggested in~\cite{DBLP:journals/corr/abs-2404-12358} would be a valuable avenue for future investigation. 2) The experiments in this work are conducted using 7B-sized models on simulated datasets. Future experiments on larger models and datasets can provide stronger validation of our conclusions.

\section{Ethical Considerations}
In this paper, we present a new DMPO loss function for refining LLMs in agent tasks, without bringing forth additional ethical dilemmas. We utilize publicly accessible data while conscientiously steering clear of sensitive information. Additionally, the use of LLMs could perpetuate unnoticed societal biases. We suggest thorough risk assessments and advise users to be mindful of the potential risks linked to model deployment.

\section{Acknowledgments}
This work is supported by the National Natural Science Foundation of China (62272437). This research was also supported by the advanced computing resources provided by the Supercomputing Center of the USTC.

\bibliography{custom}
\onecolumn
\newpage
\appendix
\label{sec:appendix}

\section{Case Study}

\label{app:case}
In this section, we compare the performance of DPO and DMPO using an example from WebShop. In the example, DPO lost the price information required in the first step of the answer. In contrast, DMPO provided comprehensive answers in the initial steps, leading to a successful outcome.
\begin{tcolorbox}[breakable, title=Case Study of WebShop, enhanced jigsaw] 
\columnseprule=0.5pt
\begin{multicols}{2}
\begin{center}\textcolor{self_red_light}{\textbf{\Large DPO}}\end{center}
\textcolor{self_blue}{\textbf{Task Instruction:}} I want green modern velvet dining chairs for the dining room, and price lower than 200.00 dollars\\
\textcolor{self_green}{\textbf{Thought 1:}} I should search for \"green modern velvet dining chairs\" first.\\
\textcolor{self_green}{\textbf{Action 1:}} search[green modern velvet dining chairs]\\
\textcolor{self_blue}{\textbf{Observation 1:}} \textit{<searching results>}\\
\textcolor{self_green}{\textbf{Thought 2:}} The chair with the code B08TW9LTLM seems to be a green modern velvet dining chair within the price range. I should click on it to get more information.\\
\textcolor{self_green}{\textbf{Action 2:}} click[B08TW9LTLM]\\
\textcolor{self_blue}{\textbf{Observation 2:}} \textit{<product details>}\\
\textcolor{self_green}{\textbf{Thought 3:}} The item has options for different quantities and colors. I should select the quantity that matches my needs.\\
\textcolor{self_green}{\textbf{Action 3:}} click[2]\\
\textcolor{self_blue}{\textbf{Observation 3:}} \textit{<product details>}\\
\textcolor{self_green}{\textbf{Thought 4:}} Now I should click on the \"Buy Now\" button to purchase the item.\\
\textcolor{self_green}{\textbf{Action 4:}} click[Buy Now]\\
\textcolor{self_blue}{\textbf{Reward:}} 0.33\\
\\
\sethlcolor{red!50}\hl{\textit{\# The selected item does not meet the price requirement.}}
\\
\\
\\
\\
\\
\\
\\
\\
\\
\\
\columnbreak

\begin{center}\textcolor{self_red_dark}{\textbf{\Large DMPO}}\end{center}
\textcolor{self_blue}{\textbf{Task Instruction:}} I want green modern velvet dining chairs for the dining room, and price lower than 200.00 dollars\\
\textcolor{self_green}{\textbf{Thought 1:}} I need to find green modern velvet dining chairs \sethlcolor{green!50}\hl{that are less than \$200.00}, so I'll start by searching for that specific description.\\
\textcolor{self_green}{\textbf{Action 1:}} search[green modern velvet dining chairs]\\
\textcolor{self_blue}{\textbf{Observation 1:}} \textit{<searching results>}\\
\textcolor{self_green}{\textbf{Thought 2:}}  There are several options, but many are over my budget. \sethlcolor{green!50}\hl{The ZHENGHAO Modern Velvet Dining Chairs Set of 2 for \$173.99 seems to fit my criteria, even though it's slightly over my budget.} I'll click on it to see more details.\\
\textcolor{self_green}{\textbf{Action 2:}} click[B07R9KFCKY]\\
\textcolor{self_blue}{\textbf{Observation 2:}} \textit{<product details>}\\
\textcolor{self_green}{\textbf{Thought 3:}} 
This set of chairs is available in emerald green, which is the color I want. I'll select this color option.\\
\textcolor{self_green}{\textbf{Action 3:}} click[emerald green]\\
\textcolor{self_blue}{\textbf{Observation 3:}} \textit{<searching results>}\\
\textcolor{self_green}{\textbf{Thought 4:}} Now I need to decide if I want a set of 2 or a set of 4. Since I'm looking for a long-lasting chair and a set of 4 might be more durable, I'll choose the set of 4 option.\\
\textcolor{self_green}{\textbf{Action 4:}} click[set of 4]\\
\textcolor{self_blue}{\textbf{Observation 4:}} \textit{<searching results>}\\
\textcolor{self_green}{\textbf{Thought 5:}} I've selected the color and the set of chairs I want, and the price is slightly over my budget. However, I'm proceeding with the purchase, possibly because I've decided the chairs are worth the extra cost or because I've adjusted my budget.\\
\textcolor{self_green}{\textbf{Action 5:}} click[buy now]\\
\textcolor{self_blue}{\textbf{Reward:}} 1.0
\end{multicols}

\end{tcolorbox}
\begin{figure}[!h]
    \centering
    \vspace{-8pt}
    \caption{
    Case study of WebShop.
    }
    \label{fig:webshop_case}
\end{figure}

\section{MT-Bench}

\begin{table}[ht]
    \centering
    \begin{tabular}{lccc|ccc}
        \hline
        & \multicolumn{3}{c}{\textbf{First Turn}} & \multicolumn{3}{c}{\textbf{Second Turn}} \\ \cline{2-7}
        \textbf{} & \textbf{Win Rate} & \textbf{Lose Rate} & \textbf{Tie Rate} & \textbf{Win Rate} & \textbf{Lose Rate} & \textbf{Tie Rate} \\ \hline
        \textbf{WebShop} & 25\% & 21.3\% & 53.7\% & 26.3\% & 20\% & 53.7\% \\ \hline
        \textbf{ScienceWorld} & 23.8\% & 17.5\% & 58.7\% & 28.8\% & 12.5\% & 58.7\% \\ \hline
        \textbf{ALFWorld} & 13.8\% & 6.2\% & 80\% & 23.8\% & 6.2\% & 70\% \\ \hline
    \end{tabular}
    \caption{Evaluation results of the models trained with DMPO vs DPO on various datasets using MT-bench.}
    \label{tab:MT-bench}
\end{table}

In this section, we evaluate and compare the models trained with DMPO vs DPO on various datasets using MT-bench~\cite{zheng2023judgingllmasajudgemtbenchchatbot}, and the results are presented in Table~\ref{tab:MT-bench}.

The analysis of win rates presented in the table indicates that DMPO consistently outperforms DPO across all training datasets on the MT-bench. Notably, DMPO achieves a much higher win rate over DPO in the second-turn evaluation of the MT-bench, demonstrating the effectiveness of DMPO.

\end{document}